%% file: ms.tex
\pdfoutput=1
\documentclass[sigconf]{aamas}  %

\renewcommand\footnotetextcopyrightpermission[1]{} %
\input{preamble_arXiv}
\usepackage{enumitem}
\usepackage{epstopdf}

\makeatother

\begin{document}

\title{GANGs: Generative Adversarial Network Games}

\author{Frans A. Oliehoek}
\affiliation{\institution{University of Liverpool}}
\email{frans.oliehoek@liverpool.ac.uk}

\author{Rahul Savani}
\affiliation{\institution{University of Liverpool}}
\email{rahul.savani@liverpool.ac.uk}

\author{Jose Gallego-Posada}
\affiliation{\institution{University of Amsterdam}}
\email{jose.gallegoposada@student.uva.nl}

\author{Elise van der Pol}
\affiliation{\institution{University of Amsterdam}}
\email{e.e.vanderpol@uva.nl}

\author{Edwin D. {de Jong}}
\affiliation{\institution{IMC Financial Markets}}
\email{edwin-de-jong.github.io}

\author{Roderich Gro{\ss}}
\affiliation{\institution{The University of Sheffield}}
\email{r.gross@sheffield.ac.uk}


\begin{abstract}
Generative Adversarial Networks (GAN) have become one of the most
successful frameworks for unsupervised generative modeling. As
GANs are difficult to train much research has focused on this. However,
very little of this research has directly exploited game-theoretic
techniques. We introduce Generative Adversarial Network Games (GANGs), which
explicitly model %
a finite zero-sum game between
a generator ($G$) and classifier ($C$) that use \emph{mixed}
strategies. %
The size of these games precludes exact solution methods, therefore 
we define resource-bounded best responses
(RBBRs), and a resource-bounded Nash Equilibrium (RB-NE) as a pair of mixed
strategies such that neither $G$ or $C$ can find a better RBBR. The RB-NE
solution concept is richer than the notion of `local Nash equilibria'
in that it captures not only failures of escaping local optima of
gradient descent, but applies to any approximate best response computations,
including methods with random restarts. To validate our approach,
we solve GANGs with the Parallel Nash Memory algorithm, which provably
monotonically converges to an RB-NE. We compare our results to standard
GAN setups, and demonstrate that our method deals well with typical
GAN problems such as mode collapse, partial mode coverage and forgetting. 
\end{abstract}

\maketitle

\input{notation.tex}
\vspace{-0.5cm}

\input{GANGs_arXiv--1-intro.tex}

\input{GANGs_arXiv--2-background.tex}

\input{GANGs_arXiv--3-GANGs.tex}

\input{GANGs_arXiv--3b-RB-GANGs.tex}
\input{GANGs_arXiv--4-solving-GANGs.tex}

\input{GANGs_arXiv--5-experiments.tex}

\input{GANGs_arXiv--related-work.tex}

\input{GANGs_arXiv--6-conclusions.tex}

\section*{Acknowledgments}
This research made use of a GPU donated by Nvidia.

%

\input{ms.bbl}

\end{document}

%% file: preamble_arXiv.tex
\usepackage{booktabs}

\setcopyright{ifaamas}  %
\acmDOI{doi}  %
\acmISBN{}  %
\acmConference[AAMAS'18]{Proc.\@ of the 17th International Conference on Autonomous Agents and Multiagent Systems (AAMAS 2018), M.~Dastani, G.~Sukthankar, E.~Andre, S.~Koenig (eds.)}{July 2018}{Stockholm, Sweden}  %
\acmYear{2018}  %
\copyrightyear{2018}  %
\acmPrice{}  %

\usepackage{subfig}
\usepackage{array}
\graphicspath{  {.}
                {fig1/figs/}
                {fig1b/figs/}
                {fig2/figs/}
                {fig3/figs/}
                {gan_no_bn_appendix/figs/}
            }

\usepackage{algorithm}
\usepackage{algpseudocode}

\usepackage{amsfonts}       %
\usepackage{amsmath}
\usepackage{amssymb}

\usepackage{verbatim}
\usepackage{float}
\usepackage{amsthm}
\usepackage{amstext}

\makeatletter

\pdfpageheight\paperheight
\pdfpagewidth\paperwidth

\floatstyle{ruled}
\newfloat{algorithm}{tbp}{loa}
\providecommand{\algorithmname}{Algorithm}
\floatname{algorithm}{\protect\algorithmname}

\theoremstyle{plain}
\newtheorem{thm}{\protect\theoremname}
\theoremstyle{definition}
\newtheorem{defn}[thm]{\protect\definitionname}
\theoremstyle{plain}
\newtheorem{fact}[thm]{\protect\factname}
\ifx\proof\undefined
\newenvironment{proof}[1][\protect\proofname]{\par
\normalfont\topsep6\p@\@plus6\p@\relax
\trivlist
\itemindent\parindent
\item[\hskip\labelsep\scshape #1]\ignorespaces
}{%
\endtrivlist\@endpefalse
}
\providecommand{\proofname}{Proof}
\fi

\input{notation.general.tex}
\input{notation.references-lyx.tex}

\input{notation.games.tex}
\usepackage{verbatim}

\providecommand{\definitionname}{Definition}
\providecommand{\factname}{Fact}
\providecommand{\theoremname}{Theorem}

%% file: notation.general.tex
\providecommand{\defas}     {\mathrel{\triangleq}}

\providecommand{\funcName}[1]{\textsc{#1}}
\newcommand{\E}{\mathbf{E}}

\providecommand{\reals}{\mathbb {R} }

\newcommand{\set}[1]{\mathcal{#1}}

\input{notation.argsPlacement.time_sup.tex}

%% file: notation.argsPlacement.time_sup.tex
\providecommand{\indexMarkup}[1]{(#1)}

 \providecommand{\argsA}[2]{ {#1}_{#2} } 
 \providecommand{\argsAI}[3]{ {#1}_{#2}^{\indexMarkup{#3}} }


%% file: notation.references-lyx.tex
\providecommand{\fig}{Figure~}

\providecommand{\rthm}{Theorem~}

%% file: notation.games.tex
\providecommand{\AGS}{D}
\providecommand{\agentS}{\set{\AGS}}
\providecommand{\agentI}[1]{{#1}}
\providecommand{\nrA}{n} %

\providecommand{\AC}{s}                 %
\providecommand{\ACS}{\set{S}}                     %

\providecommand{\aA}[1]{\argsA{\AC}{#1}}
\providecommand{\aAS}[1]{\argsA{\ACS}{#1}}      
\providecommand{\aAI}[2]{\argsAI{\AC}{#1}{#2}}

\providecommand{\utF}{u}		%
\providecommand{\utA}[1]{\argsA{\utF}{#1}}	%

%% file: notation.tex
\section*{}

\global\long\def\nrA{n}

\global\long\def\aA#1{\argsA{\AC}{#1}}

\global\long\def\utA#1{\argsA{\utF}{#1}}

\global\long\def\mA#1{\argsA{\mu}{#1}}

\global\long\def\aAI#1#2{\argsAI{\AC}{#1}{#2}}

\global\long\def\aAS#1{\argsA{\ACS}{#1}}

\global\long\def\gameval{v^{*}}

\global\long\def\mf{\phi}
 %

\global\long\def\paramG{\theta_{G}}
 %

\global\long\def\paramC{\theta_{C}}
 %

\global\long\def\paramGS{\Theta_{G}}
 %

\global\long\def\paramCS{\Theta_{C}}
 %

\global\long\def\dim{d}
 %

\input{notation.games.tex}

%% file: GANGs_arXiv--1-intro.tex
\section{Introduction}

Generative Adversarial Networks (GANs) \citep{Goodfellow14NIPS27}
are a framework in which two neural networks compete with each other:
the \emph{generator (G) }tries to trick the \emph{classifier (C)}
into classifying its generated fake data as true. GANs hold great
promise for the development of accurate generative models for
complex distributions, such as the distribution of images of written
digits or faces. Consequently, in just a few years, GANs have grown
into a major topic of research in machine learning. A core appeal
is that they do not need to rely on distance metrics~\citep{Li16SI}.
However, GANs are difficult to train and much research has focused
on this \citep{Unterthiner17arXiv,ArjovskyB17,ACB17_WGAN}. Typical
problems are \emph{mode collapse} in which the generator only outputs
points from a single mode
of the distribution and \emph{partial mode coverage
} in which the generator only learns to represent a (small) subset
of the modes of the distribution. Moreover, while learning the players
may \emph{forget}: e.g., it is possible that a
classifier correctly learns to classify part of the input space as
`fake' only to forget this later when the generator no longer generates
examples in this part of space. 
Finally, when learning via gradient descent one can get stuck 
in \emph{local} Nash equilibria~\citep{Ratliff13ACCCC}.

We introduce a novel approach that does not suffer from local equilibria:
\emph{Generative Adversarial Network Games (GANGs)} formulate adversarial networks
as \emph{finite} zero-sum games, and the solutions that
we try to find are saddle points in \emph{mixed strategies}. This approach is
motivated by the observation that, considering a GAN as a finite zero-sum game,
in the space of mixed strategies, any local Nash equilibrium is a global one.
Intuitively, the reason for this is that whenever there is a profitable pure
strategy deviation one can move towards it in the space of mixed strategies.

However, as we cannot expect to find such exact best responses due to the extremely
large number of pure strategies that result for sensible choices of neural
network classes, we introduce Resource-Bounded Best-Responses (RBBRs), and the
corresponding Resource-Bounded Nash equilibrium (RB-NE), which is a pair of mixed
strategies in which no player can find a better RBBR. This is richer than the
notion of local Nash equilibria in that it captures not only failures of
escaping local optima of gradient descent, but applies to \emph{any} approximate best
response computations, including methods with random restarts, and allows us to
provide convergence guarantees. 

The key features of our approach are that:
\begin{itemize}[leftmargin=3mm]
	\item It is based on a framework of finite zero-sum games, and as such it
		enables the use of existing game-theoretic methods. In this paper we 
		focus on one such method, Parallel Nash Memory (PNM).
    \item We show that PNM will provably and monotonically converge to an RB-NE. 
	\item It enables us to understand existing GAN objectives 
		(WGAN, Goodfellow's training heuristic) in the context of zero-sum games.
	\item Moreover, it works for any network architectures (unlike previous
		approaches, see Related Work). In particular, future improvements in
		classifiers/generator networks can directly be exploited.
\end{itemize}

We investigate empirically the effectiveness of PNM and show that it can indeed
deal well with typical GAN problems such as mode collapse, partial mode coverage
and forgetting, especially in distributions with less symmetric structure to
their modes. 
At the same time, a naive implementation of our method does have some
disadvantages and we provide an intuitively appealing way to deal with these.
Our promising results suggest several interesting directions for further work.

%% file: GANGs_arXiv--2-background.tex
\section{Background }

\label{sec:background}  
We defer a more detailed treatment of related work on GANs and recent game theoretic approaches until the end of the paper.
Here, we start by introducing some basic game-theoretic notation.
\begin{defn}
    A \emph{strategic game} (also `normal-form game'), is a tuple $\left\langle \agentS,\left\{ \aAS i\right\} _{i=1}^{\nrA},\left\{ \utA i\right\} _{i=1}^{\nrA}\right\rangle $,
where $\agentS=\{\agentI1,\dots,\agentI\nrA\}$ is the set of\emph{
players}, $\aAS i$ is the set of \emph{pure strategies} for player
$i$, and $\utA i:\mathcal{S}\to\mathbb{R}$ is $i's$ payoff function
defined on the set of pure strategy profiles $\mathcal{{S}}:=\aAS 1\times\dots\times\aAS{\nrA}$. 
When the set of players and their action sets are finite, the strategic game is \emph{finite}.
\end{defn}
A fundamental concept in game theory is the Nash equilibrium (NE), which
is a strategy profile such that no player can unilaterally deviate
and improve his payoff. 
\begin{defn}[Pure Nash equilibrium]
 A pure strategy profile $s=\left\langle s_{1},\dots,s_{\nrA}\right\rangle $
is an NE if and only if\emph{ $\utA i(s)\geq\utA i(\left\langle \aA 1,\dots,\aA i',\dots,\aA{\nrA}\right\rangle )$
}for all players $i$ and $\aA i'\in\aAS i$.
\end{defn}
A finite game may not possess a pure NE. A \emph{mixed strategy} of
player $i$ is a probability distribution over $i$'s pure strategies
$\aAS i$. The payoff of a player under a profile of mixed strategies
$\mA{} = \left\langle \mA{1},\dots,\mA{\nrA}\right\rangle $ is defined as
the expectation: 
\[
    \utA i(\mA{}):=
                \sum_{\aA{}\in\aAS{}} 
                [ \prod_{j\in \agentS}\mA{j}(\aA{j}) ] \cdot \utA{i}(\aA{}).
\]

Then an NE in mixed strategies is defined as follows. 
\begin{defn}[Mixed Nash equilibrium]
    A %
    $\mA{}$
    is an NE if and only if\emph{ $\utA i(\mA{})\geq\utA i(\left\langle \mA{1},\dots,\aA i',\dots,\mA{n}\right\rangle )$}
for all players $i$ and $\aA i'\in\aAS i$.
\end{defn}
Every finite game has at least one NE in mixed strategies~\cite{Nash50}.
In this paper we deal with two-player %
\emph{zero-sum} games, where
$\utA 1(\aA 1,\aA 2)=-\utA 2(\aA 1,\aA 2)$ for all $\aA 1\in\aAS 1,\aA 2\in\aAS 2$.
The equilibria of zero-sum games, also called \emph{saddle points},
have several important properties,
as stated in the following theorem.
\begin{thm}[Minmax Theorem \cite{vNM28}]
\label{thm:minmax}In a zero-sum game, we have 
\[
    \min_{\mA{1}}\max_{\mA{2}}u_{i}(\mA{})=\max_{\mA{2}}\min_{\mA{1}}u_{i}(\mA{})=v.
\]
We call $v$ the value of the game. All equilibria have payoff $v$.
\end{thm}

Moreover, these equilibria $\mA{}$ can be expressed in terms of so-called 
maxmin strategies.
A \emph{maxmin strategy} of player 1 is a $\hat{{\mA{1}}}$ that solves
$min_{\mA{2}}u_{1}(\hat{{\mA{1}}},\mA{2})=v$, and a maxmin strategy
of player 2 is a $\hat{{\mA{2}}}$ that solves $min_{\mA{1}}u_{2}(\mA{1},\hat{{\mA{2}}})=v$.
Any pair of maxmin strategies of the players is an equilibrium. This 
directly implies that equilibrium strategies are interchangeable:
if
$ \langle \mA{1}, \mA{2} \rangle $
and
$ \langle \mA{1}',\mA{2}' \rangle $
are equilibria, then so are
$ \langle \mA{1}',\mA{2} \rangle $
and
$ \langle \mA{1} ,\mA{2}' \rangle $ \citep{Osborne+Rubinstein94}.
Moreover, the convex combination of two equilibria is an equilibrium, 
meaning that  the game has either one or infinitely many equilibria.

We will not always be able to compute an exact NE, and so we employ
the standard, additive notion of \emph{approximate equilibrium:}
\begin{defn}
    Let $\mA{-i}$ denote a strategy of the opponent of player $i$.
    A pair of (possibly pure) strategies $(\mA i,\mA{-i})$ is an $\epsilon$-NE if
\begin{equation}
    \forall i\qquad
    \utA i(\mA i,\mA{-i})
    \geq
    \max_{\mA i'}
    \utA i(\mA i',\mA{-i})-\epsilon.
\label{eq:epsilon-NE}
\end{equation}
In other words, no player can gain more than $\epsilon$ by deviating.
\end{defn}

In the literature, GANs have not typically been considered as finite
games. The natural interpretation of the standard setup of GANs is
of an infinite game where payoffs are defined over all possible weight
parameters for the respective neural networks. With this view we do
not obtain existence of saddle points, nor the desirable properties
of Theorem \ref{thm:minmax}. Some results on the existence of saddle
points in infinite action games are known, but they require properties
like convexity and concavity of utility functions \citep{Aubin98optima}, which we cannot
apply as they would need to hold w.r.t. the neural network parameters.
This is why the notion of  \emph{local Nash equilibrium (LNE)} has
arisen in the literature \citep{Ratliff13ACCCC,Unterthiner17arXiv}.
Roughly, an LNE is a strategy profile where neither player can improve
in a small neighborhood of the profile. In finite games
every LNE is an NE, as, whenever there is a global deviation, one
can always deviate locally in the space of mixed strategies 
towards a pure best response.

%% file: GANGs_arXiv--3-GANGs.tex
\section{GANGs}

\label{sec:GANGs}

Intuitively, it is clear that GANs are closely related to games, and
the original paper by \cite{Goodfellow14NIPS27} already points out
that $G$ and $C$ essentially play a minimax game. %
The exact relation has not been explicitly described, leading to confusion
on whether the sought solution is a saddle point or not\textbf{ }\citep{NIPSpanel16}.\textbf{
} We set out to address this confusion by introducing
the \emph{Generative Adversarial Network Game (GANGs)} formalism, which 
explicitly phrases adversarial networks as zero-sum strategic-form games. It
builds on the building blocks of regular GANs, but it emphasizes the
fact that $G$ and $C$ play a zero-sum game, and formalizes the action
space available to these players.

In this section, we provide a `fully rational' formulation of GANGs,
i.e., one where we make no or limited assumptions on the bounded resources
that we face in practice. Bounded resources will be introduced in
the next section.
\begin{defn}
\label{def:A-GANG-is}A \textbf{GANG }is a tuple $\mathcal{M} = \left\langle p_{d},\left\langle G,p_{z}\right\rangle ,C,\mf\right\rangle $
with 
\begin{itemize}[leftmargin=3mm]
\item $p_{d}(x)$ is the distribution over (`true' or `real') data points
$x\in\mathbb{R}^{\dim}$.
\item $G$ is a neural network class with parameter vector $\paramG\in\paramGS$
and $d$ outputs, such that $G(z;\paramG)\in\mathbb{R}^{\dim}$ denotes
the (`fake' or `generated') output of $G$ on a random vector $z$
drawn from some distribution $z\sim p_{z}$. We will typically leave
dependence on the parameter vector $\paramG$ implicit and just write
$G(z)$.
\item $C$ is a neural network class with parameter vector $\paramC\in\paramCS$
and a single output, such that the output $C(x;\paramC)\in[0,1]$
indicates the `realness' of $x$ according to $C$. We will interpret
it as the probability that $C$ assigns to a point $x$ being real, even
though it does not need to be a strict probability. 
\item $\mf$ is a \emph{measuring function} \citep{Arora17ICML}: $\mf:[0,1]\to\reals$
---typically $\log$, for GANs, or the identity mapping, for WGANs---%
that is used to specify the payoffs of the agents, explained next.
\end{itemize}
\end{defn}
A GANG induces a zero-sum game in an intuitive way:
\begin{defn}
\label{def:The-induced-zero-sum}The induced zero-sum strategic-form
game of a GANG is $\left\langle \agentS=\{G,C\},\left\{ \aAS G,\aAS C\right\} ,\left\{ \utA G,\utA C\right\} \right\rangle $
with:
\begin{itemize}[leftmargin=3mm]
\item $\aAS G=\left\{ G(\cdot;\paramG)\mid\paramG\in\paramGS\right\} $,
elements of which we denote by $\aA G$;
\item $\aAS C=\left\{ C(\cdot;\paramC)\mid\paramC\in\paramCS\right\} $,
elements of which we denote by $\aA C$;
\item The payoff of G, for all ($\paramG\in\paramGS,\paramC\in\paramCS$
and their induced) strategies $\aA G$, $\aA C$ is $\utA G(\aA G,\aA C)=-\utA C(\aA G,\aA C)$; 
\item The payoff of the classifier is given by:
\[
\utA C(\aA G,\aA C)=\E_{x\sim p_{d}}\left[\mf\left(\aA C(x)\right)\right]-\E_{z\sim p_{z}}\left[\phi\left(\aA C\left(\aA G(z)\right)\right)\right].
\]
That is, the score of the classifier is the expected correct classification
on the real data minus the expected incorrect classification on the
fake data. 
\end{itemize}
\end{defn}
In practice, GANs are represented using floating point numbers, of which, for 
a given setup, there 
is only a finite (albeit large) number. In GANGs, we formalize this:
\begin{defn}
Any GANG where $G,C$ are finite classes---i.e.,
classes of networks constructed from a finite set of node types (e.g.,
\{Sigmoid, ReLu, Linear\})---and with
architectures of bounded size, is called a \emph{finite network
class GANG.} A finite network class GANG in which the sets $\paramGS,\paramCS$
are finite is called a \emph{finite GANG. }
\end{defn}
From now on, we deal with finite GANGs. These correspond to finite
(zero-sum) strategic games and, as explained earlier, they will possess one or
infinitely many mixed NEs with the same payoff. 
Note that we only point out that finiteness of floating point systems leads to
finite GANGs, but we do not impose any additional constraints or
discretization.  Therefore, our finite GANGs have the same representational
capacity as normal GANs that are implemented using floating point arithmetic.

\textbf{Zero-sum vs Non-zero-sum. }In contrast to much of the GAN-literature,
we explicitly formulate GANGs as being zero-sum games. GANs \cite{Goodfellow14NIPS27}
formulate 
the payoff of the generator as a function of the fake data only: 
$\utA G = \E_{z\sim p_{z}}\left[\mf\left(\aA C\left(\aA G(z)\right)\right)\right]$. 
However, it turns out that this difference typically has no implications
for the sought solutions. We clarify this with the following theorem,
and investigate particular instantiations below. In game theory, 
two games are called \emph{strategically equivalent} if
they possess exactly the same set of Nash equilibria; this (standard)
definition is concerned only about the mixed strategies played in equilibrium
and not the resulting payoffs. The following is a well-known game
transformation (folklore, see \cite{Liu96}) that creates a new strategically
equivalent game:
\begin{fact}
\label{fact:equiv}Consider a game $\Gamma$=$\left\langle \{1,2\},\{\aAS 1,\aAS 2\},\{\utA 1,\utA 2\}\right\rangle $.
Fix a pure strategy $s_{2}\in\aAS 2.$ Define $\bar{{\utA 1}}$ as
identical to $\utA 1$ except that $\bar{\utA 1}(\aA{i},\aA 2)=\utA 1(\aA{i},\aA 2)+c$
for all $\aA{i}\in\aAS 1$ and some constant $c$. 
We have that $\Gamma$ and 
$\bar{{\Gamma}}$=$\left\langle \{1,2\},\{\aAS 1,\aAS 2\},\{\bar{\utA 1},\utA 2\}\right\rangle $
are strategically equivalent.
\end{fact}
\begin{thm}
\label{thm:fake-real-decomp}Any finite (non-zero-sum) two-player
game between $G$ and $C$ with payoffs of the following form:
\[
\utA G=Fake_{G}(\aA G,\aA C)=-Fake_{C}(\aA G,\aA C),
\]
\[
\utA C=Real_{C}(\aA C)+Fake_{C}(\aA G,\aA C),
\]
is strategically equivalent to a zero-sum game where $G$ has payoff
$\bar{\utA G}\defas-Real_{C}(\aA C)-Fake_{C}(\aA G,\aA C).$\end{thm}
\begin{proof}
By adding$-Real_{C}(\aA C)$ to $G$'s utility function, for each pure strategy $\aA C$ of $C$ 
we add a different constant
to all utilities of $G$ against $\aA C$. Thus, by applying
Fact \ref{fact:equiv} iteratively for all $\aA C\in\aAS C$ we see
that we produce a strategically equivalent game.
\end{proof}

Next, we formally specify the conversion of existing GAN models
to GANGs. We consider the general measure function %
that covers GANs and WGANs. In these models, the payoffs are specified
as
\[
\utA G(\aA G,\aA C)\defas-\E_{z\sim p_{z}}\left[\mf\left(1-\aA C\left(\aA G(z)\right)\right)\right],
\]
\[
\utA C(\aA G,\aA C)\defas\E_{x\sim p_{d}}\left[\mf\left(\aA C(x)\right)\right]+\E_{z\sim p_{z}}\left[\mf\left(1-\aA C\left(\aA G(z)\right)\right)\right].
\]
{\sloppy
These can be written using 
$Fake_{G}(\aA G,\aA C)=-Fake_{C}(\aA G,\aA C)=-\E_{z\sim p_{z}}\left[\mf\left(1-\aA C\left(\aA G(z)\right)\right)\right]$
and $Real_{C}(\aA C)=\E_{x\sim p_{d}}\left[\mf\left(\aA C(x)\right)\right]$.
This means that we can employ \rthm\ref{thm:fake-real-decomp} and
equivalently define 
a GANG with zero-sum payoffs that preserves the NEs.
}

In practice, most work on GANs uses a different
objective, introduced by \cite{Goodfellow14NIPS27}. They say that {[}formulas
altered{]}: ``Rather than training $G$ to minimize $\log(1-\aA C(\aA G(z)))$ we
can train $G$ to maximize $\log\aA C(\aA G(z))$. This objective function
results in the same fixed point of the dynamics
of $G$ and $C$ but provides much stronger gradients early in learning.'' 
This means that they redefine 
$\utA G(\aA G,\aA C)\defas\E_{z\sim p_{z}}\left[\mf\left(\aA C\left(\aA G(z)\right)\right)\right],$
which still can be written as 
$Fake_{G}(\aA G,\aA C)=\E_{z\sim p_{z}}\left[\mf\left(\aA C\left(\aA G(z)\right)\right)\right]$.
Now, \emph{as long as the classifier's payoff is also adapted}
we can still write the payoff functions in the form of \rthm\ref{thm:fake-real-decomp}.
That is, the trick is compatible with a zero-sum formulation, as
long as it is also applied to the classifier. This then yields the formulation of the 
payoffs as used in GANGs (in Def.~\ref{def:The-induced-zero-sum}).

%% file: GANGs_arXiv--3b-RB-GANGs.tex
\section{Resource-Bounded GANGs}

\label{sec:RB-GANGs}

While GANGs clarify the relation of existing adversarial network models to
zero-sum games, we will not be able to solve them exactly. 
Even though they are
finite, the number of pure strategies will be huge when we use reasonable neural
network classes and parameter sets.
This means that finding an NE, or even an $\epsilon-$NE
will be typically beyond our capabilities.
Thus we need to deal with players with bounded computational resources. 
In this section, we propose a formal notion of such `bounded rationality`.

\textbf{Resource-Bounded Best-Responses (RBBR).} 
A Nash equilibrium is defined by the absence of better responses for any of the
players. As such, best response computation is a critical tool to verify whether
a strategy profile is an NE, and is also common subroutine in algorithms that
compute an NE. 
However, like the computation of an ($\epsilon$-)NE, computing an
($\epsilon$-)best response will generally be intractable for GANGs.
Therefore, to clarify the type of solutions that we actually can expect to
compute with our bounded computational power, we formalize the notion of
\emph{resource-bounded best response} and show how it naturally leads to 
a solution concept that we call \emph{resource-bounded NE (RB-NE)}.
\begin{defn}
We say that $\aAS i^{RB}\subseteq\aAS i$ is the subset of strategies
of player $i$, that $i$ can compute as a best response, given
its bounded computational resources.
\end{defn}
This computable set is an abstract formulation to capture phenomena
like gradient descent being stuck in local optima, but also more general
reasons for not computing a best response, such as not even reaching
a local optimum in the available time. 
\begin{defn}
    \label{def:RBBR}
A strategy $s_{i}\in\aAS i^{RB}$ of player~$i$ is a \emph{resource-bounded
best-response} (RBBR) against a (possibly mixed) strategy $s_{j}$, if $\forall\aA i'\in\aAS i^{RB}\quad\utA i(\aA i,\aA j)\geq\utA i(\aA i',\aA j)$.
\end{defn}
That is, $\aA i$ only needs to be amongst the best strategies
that player~$i$ \emph{can compute} in response to $\aA j$. We denote
the set of such RBBRs to $\aA j$ by
$\aAS i^{RBBR(\aA j)}\subseteq\aAS i^{RB}$. 
\begin{defn}
A \emph{resource-bounded best-response function}\textbf{ $f_{i}^{RBBR}:\aAS j\to\aAS i^{RB}$
}is a function that maps from the set of possible strategies of player $j$ to an RBBR for~$i$, s.t.\ $\forall_{\aA j}\; f_{i}^{RBBR}(\aA j)\in\aAS i^{RBBR(\aA j)}$.
\end{defn}
Now we can define a class of GANGs which are a better model for games
played by neural networks:
\begin{defn}
\sloppypar
A \emph{resource-bounded GANG} is a tuple $\langle \mathcal{M},\{ f_{g}^{RBBR},f_{c}^{RBBR}\} \rangle$
with $\mathcal{M}$ a finite GANG as above, and $f_{g}^{RBBR},f_{c}^{RBBR}$
the \emph{RBBR functions for both players.}
\end{defn}

For these games, we define an intuitive specialization of NE: %

\begin{defn}
A strategy profile $s=\langle \aA i,\aA j \rangle $ is a \emph{Resource-Bounded Nash Equilibrium (RB-NE)} 
iff 
$\forall i\quad\utA i(\aA i,\aA j)\geq\utA i(f_{i}^{RBBR}(\aA j),\aA j).$
\end{defn}
That is, an RB-NE can be thought of as follows:
we present $\aA{}$ to each player $i$ and it gets the chance to
switch to another strategy, for which it can apply its bounded resources 
(i.e., use $f_{i}^{RBBR}$) exactly once. After this application,
the player's resources are exhausted and if the found $f_{i}^{RBBR}(\aA
j)$ does not lead to a higher payoff it will not have an incentive to
deviate.

Clearly, an RB-NE can be linked to the familiar notion of $\epsilon$-NE
by making assumptions on the power of the best response computation.

\begin{thm}
\label{thm:epsNashGivenSufficientResources}
If the players are powerful enough to compute an $\epsilon$-best
response function, then an RB-NE for their game is an $\epsilon$-NE.
\vspace{-2mm}
\begin{proof}
    \sloppypar
Starting from the RB-NE $(\aA i,\aA j)$, assume an arbitrary~$i$.
By definition of RB-NE $\utA i(\aA i,\aA j)\geq \utA i(f_{i}^{RBBR}(\aA j),\aA j)\geq\max_{\aA i'}\utA i(\aA i',\aA j)-\epsilon$.\end{proof}
\end{thm}

\textbf{Non-deterministic Best Responses. }The above definitions 
assumed deterministic %
RBBR functions
\textbf{$f_{i}^{RBBR}$}. However, in many cases the RBBR function can be
non-deterministic (e.g., due to random restarts),
which means that the sets $\aAS i^{RB}$ are non-deterministic.
This is not a fundamental problem, however, and the same approach can
be adapted to allow for such non-determinism. In particular, now let
$f_{i}^{RBBR}$ be a non-deterministic function, and \emph{define}
$\aAS i^{RB}$ as the (non-deterministic) range of this function.
That is, we define $\aAS i^{RB}$ as that set of strategies that
our non-deterministic RBBR function delivers. Given this modification
the definition of the RB-NE remains unchanged:
a strategy profile $s=\langle \aA i,\aA j \rangle $ is a
\emph{non-deterministic RB-NE} 
if 
each player $i$ uses all its computational resources by calling 
$f_{i}^{RBBR}(\aA j)$ once, 
and no player finds a better strategy to switch to.

%% file: GANGs_arXiv--4-solving-GANGs.tex
\section{Solving GANGs}

\label{sec:SolvingGANGs}

The treatment of GANGs as finite games in mixed strategies opens up
the possibility of extending many of the existing tools and algorithms
for these classes of games \citep{Fudenberg98book,RakhlinS13a,Foster16NIPS}.
In this section, we consider the use of the Parallel Nash Memory,
which is particularly tailored to finding approximate NEs with small
support, and which monotonically converges to such an equilibrium
\citep{Oliehoek06GECCO}.

\begin{algorithm}
\input{algo_PNM.tex}

\protect\caption{\funcName{Parallel Nash Memory for GANGs}}

\label{alg:PNM}
\end{algorithm}

\input{figs/fig_without-unifake.tex}

Here we give a concise description of a slightly simplified form of
Parallel Nash Memory (PNM) and how we apply it to GANGs.%
\footnote{This ignores some aspects of PNM that we do not use, such as ways to `discard'
old strategies~\cite{FiciciP03} that have not been used for a long time.
} 
For ease of explanation, we focus on the setting with deterministic 
best responses.\footnote{
    In our experiments, we use random initializations for the best responses.
    To deal with this non-determinism we simply ignore any tests that are not able
    to find a positive payoff over the current mixed strategy NE $\left\langle \mA G,\mA C\right\rangle $,
    but we do not terminate. Instead, we run
    for a pre-specified number of iterations. 
}
The algorithm is shown in Algorithm~\ref{alg:PNM}.
Intuitively, PNM incrementally grows a strategic game $SG$, over
a number of iterations, using the \funcName{AugmentGame} function.
It also maintains a mixed strategy NE $\left\langle \mA G,\mA C\right\rangle $
\emph{of this smaller game} at all times. In each of the iterations
it uses a `search' heuristic to deliver new promising strategies.
In our GANG setting, we use the resource-bounded best-response (RBBR)
functions of the players for this purpose. After having found new strategies,
the game is augmented with these and solved again to find a new NE
of the sub-game $SG$.

In order to augment the game, 
PNM evaluates (by simulation) each newly found strategy for each player
against all of the existing strategies of the other player, thus
constructing a new row and column for the maintained payoff matrix.

In order to implement the best response functions, any existing neural
network architectures (e.g., for GANs) can be used. However, we need
to compute RBBRs against \emph{mixtures }of networks of the other
player.  For $C$ this is trivial: we can just generate a batch
of fake data from the mixture $\mA G$. Implementing a RBBR for $G$
against $\mA C$ is slightly more involved, as we need to back-propagate
the gradient from all the different $\aA C\in\mA C$ to $G$. In our
implementation we subsample a number of such pure $\aA C$ and construct
a larger network for that purpose. Better ways of doing this are an
interesting line of future work.

Intuitively, it is clear that PNM converges to an RB-NE, which we now prove formally.
\begin{thm}
If PNM terminates, then it has found an RB-NE. %
\begin{proof}
\sloppypar
We show that $\utA{BRs}\leq0$ implies we found an RB-NE:
\begin{eqnarray}
\utA{BRs} & = & 
    \utA G(f_{G}^{RBBR}(\mA C),\mA C)   +   \utA C(\mA G,f_{C}^{RBBR}(\mA G))
    \nonumber\\
 & \leq & 0=\utA G(\mA G,\mA C)+\utA C(\mA G,\mA C)     \label{eq:UBR}
\end{eqnarray}
Note that, per Def.~\ref{def:RBBR}, 
$\utA G(f_{G}^{RBBR}(\mA C), \mA C)    \geq   
 \utA G(\aA{G}'            , \mA C) $
 for all computable $\aA{G}' \in  \aAS{G}^{RB}$ (and similar for $C$). 
Therefore, the only way that 
$\utA G(f_{G}^{RBBR}(\mA C),\mA C)  \geq
 \utA G(\mA G              ,\mA C)  $
 could fail to hold,  is if $\mA G$ would include some strategies that are not
 computable (not in $\aAS{G}^{RB}$) that provide higher payoff. However, as the
 support of $\mA G$ is composed of computed (i.e., computable) strategies in
 previous iterations, this cannot be the case.
 As such we conclude 
 $\utA G(f_{G}^{RBBR}(\mA C), \mA C)  \geq      \utA G(\mA G              ,\mA C)  $
 and similarly
 $\utA C(\mA G, f_{C}^{RBBR}(\mA G))  \geq      \utA C(\mA G              ,\mA C)  $.
 Together with \eqref{eq:UBR} this directly implies 
 $\utA G(\mA G,\mA C) = \utA G(f_{G}^{RBBR}(\mA C),\mA C)$ and 
 $\utA C(\mA G,\mA C) = \utA C(\mA G,f_{C}^{RBBR}(\mA G))$, indicating we found
 an RB-NE.
\end{proof}
\end{thm}
When there are only finitely many pure best responses that we can compute, 
as for finite GANGs, the algorithm will terminate. Moreover,
using the same finiteness argument, one can show that this method
\emph{monotonically} converges to an equilibrium~\citep{Oliehoek06GECCO}.

%% file: algo_PNM.tex
\providecommand{\commentSymb}{//}
\begin{algorithmic}[1]
\small
\State{$\langle \aA{G}, \aA{C} \rangle \gets \funcName{InitialStrategies}()                                                  $}
\State{$\langle \mA{G}, \mA{C} \rangle \gets \langle \{\aA{G}\}, \{\aA{C}\} \rangle    $} \Comment{set initial mixtures}
\State{\commentSymb{Main loop:}}
\While{not done}
    \State{$ \aA{G} = \funcName{RBBR}( \mA{C} )                                 $}  \Comment{get new bounded best resp.}
    \State{$ \aA{C} = \funcName{RBBR}( \mA{G} )                                                                                 $}
    \State \commentSymb{expected payoffs of these `tests' against mixture:}
    \State{$ u_{BRs} = \utA{G}(\aA{G},  \mA{C} ) + \utA{C}(\mA{G}, \aA{C}) $}  
    \If{$ u_{BRs} \leq 0 $}
        \State{ done $\gets$ True}
    \Else
        \State{$SG \gets \funcName{AugmentGame}(SG, \aA{G}, \aA{C})     $}
        \State{$\langle \mA{G}, \mA{C} \rangle \gets \funcName{SolveGame}(SG)                                                  $}
    \EndIf
\EndWhile
\State{\textbf{return} $\langle \mA{G}, \mA{C} \rangle $} \Comment{found an BR-NE}
\end{algorithmic}

%% file: figs/fig_without-unifake.tex
\begin{figure*}[t]
\centering
\begin{tabular}{ccccc}
\subfloat{\includegraphics[width = 0.38\columnwidth]{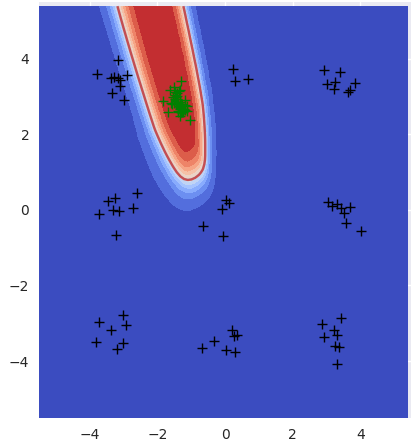}} &
\subfloat{\includegraphics[width = 0.38\columnwidth]{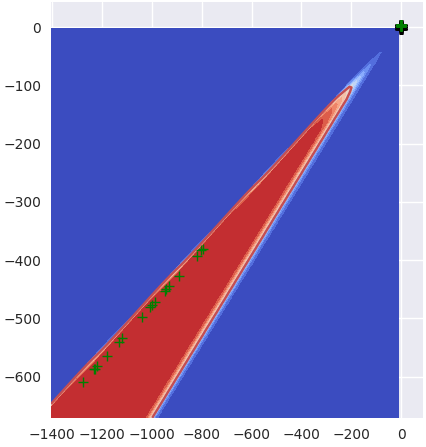}} &
\subfloat{\includegraphics[width = 0.38\columnwidth]{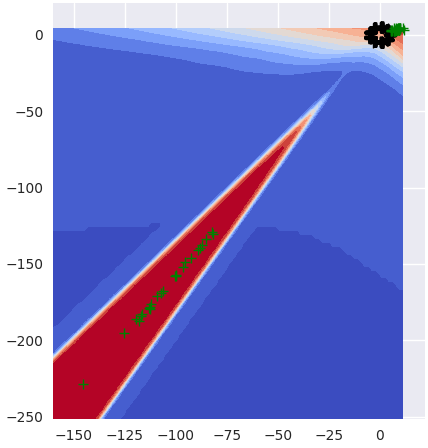}} &
\subfloat{\includegraphics[width = 0.38\columnwidth]{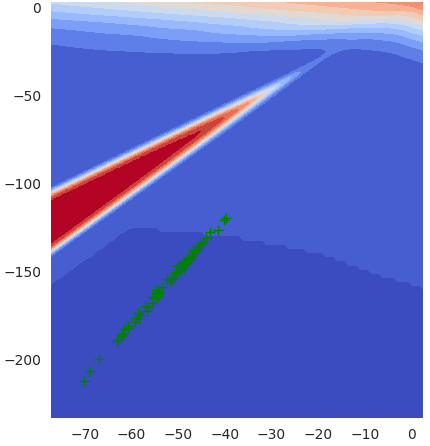}} &
\subfloat{\includegraphics[width = 0.38\columnwidth]{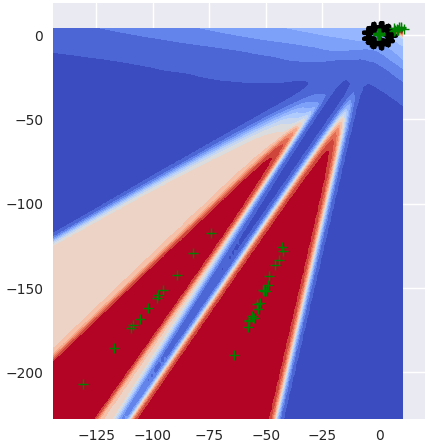}} 
\end{tabular}
\caption{Illustration of exploitation of `overfitted' classifier best-response.}
\label{fig:overfitted}
\end{figure*}

%% file: GANGs_arXiv--5-experiments.tex
\section{Experiments}
\input{figs/fig_unifake_convergence.tex}
%
\input{figs/fig2_main_comparison.tex}
\input{figs/fig3_main_comparison.tex}
\label{sec:experiments}
Here we report on experiments that aim to test if using existing
tools such as PNM can help in reducing problems with training GANs,
such as missed modes. 
This is very difficult to asses on complex data like images; 
in fact, there is debate about whether GANs are overfitting (memorizing the data). Assessing this from samples is very difficult; only crude methods have been proposed e.g., \cite{AroraZ17}.
Therefore, we restrict our proof-of-concept results to synthetic mixtures of
Gaussians for which the distributions can readily be visualized.

\noindent\textbf{Experimental setup. }We compare to a GAN implementation by
\citet{shaform}. For PNM, we use the same architectures for $G$ and
$C$ as the GAN implementation. The settings for GAN and PNM training are summarized in Table \ref{tab:settings}.
\begin{table}[h]
\centering
\scalebox{0.8}{
\begin{tabular}{c|c|c|}
\cline{2-3}
 & \textbf{GAN} & \textbf{RBBR} \\ \hline
\multicolumn{1}{|c|}{\textit{Iterations}} & 2500 & 1000 \\ \hline
\multicolumn{1}{|c|}{\textit{Learning Rate}} & $2\cdot 10^{-4}$ & $10^{-3}$ \\ \hline
\multicolumn{1}{|c|}{\textit{Batch Size}} & 64 & 128 \\ \hline
\multicolumn{1}{|c|}{\textit{Measuring Function}} & $\log$ & $10^{-5}$-bounded $\log$ \\ \hline
\end{tabular}
}\\
~
\caption{Settings used to train GANs and RBBRs.}
\vspace{-5mm}
\label{tab:settings}
\end{table}
The mixture components comprise grids and annuli with equal-variance components, as well as non-symmetric cases with randomly located modes and with a random covariance matrix for each mode.
For each domain we create test cases with 9 and 16
 components. In our plots,
black points are real data, green points are generated data. 
Blue indicates areas that are classified as `realistic'
while red indicates a `fake' classification by~$C$.

\noindent
\textbf{Plain application of PNM. }
In GANGs, $G$ informs $C$ about what good strategies are and vice versa.
However, as we will make clear here, this $G$
has limited incentive to provide the best possible training signal
to $C$.
This is illustrated in \fig~\ref{fig:overfitted}. The leftmost
two plots show the \emph{same }best response by $C$: zoomed in on the
data and zoomed out to cover some fake outliers. Clearly, $C$ needs
to really find creative solutions to try and get rid of the faraway
points, and also do good near the data. As a result, it ends up with
the shown narrow beams in an effort to give high score to the true
data points (a very broad beam would lower their scores), but this
exposes $C$ to being exploited in later iterations: $G$ needs to merely
shift the samples to some other part of the vast empty space around
the data. This phenomenon is nicely illustrated by the remaining three
plots (that are from a different training run, but illustrate it well):
the middle plot shows an NE that targets one beam, this is exploited
by $G$ in its next best response (fourth image, note the different scales
on the axes, the `beam' is the same). The process continues, and
$C$ will need to find mixtures of all these type of complex counter
measures (rightmost plot). This process can take a long time.

\noindent\textbf{PNM with Added Uniform Fake Data.} However, the GANG formalism
allows us to incorporate a simple way to resolve this
issue and make training more effective. In each iteration, we look
at the total span (i.e., bounding box) of the real and fake data,
and we add some uniformly sampled fake data in this bounded
box (we used the same amount as fake data produced by $G$). In that
way, we further guide $C$ in order to better guide the
generator (by directly making clear that all the area beyond the true
data is fake). The impact of this procedure is illustrated by \fig~\ref{fig:unifakeconv},
which shows the payoffs that the maintained mixtures $\mA G,\mA C$
achieve against the RBBRs computed against them (so this is a measure
of security), as well as the `payoff for tests' ($\utA{BR}$). %
Clearly, adding uniform fake data leads
to much faster convergence.\\~%
As such, we perform our main comparison to GANs with this uniform
fake data component added in. These results are shown in \fig~\ref{fig:9comp} and
\ref{fig:16comp}, and clearly convey three main points: first, the PNM mixed classifier
has a much flatter surface than the classifier found by the GAN, which
is in line with the theoretical predictions about the equilibrium
\citep{Goodfellow14NIPS27}. More importantly, however, we see that
this flatter surface is not coming at the cost of inaccurate samples.
In contrast: nearly all samples shown are hitting one of the modes
and thus are highly accurate, much more so than the GAN's samples.
Finally, except in the 16-component grid plot, we see that the approach
does not suffer from partial mode coverage as it leaves out no modes. We point out
that these PNM results are without batch normalization, while the GAN results without 
batch normalization suffered from many inaccurate samples and severe
partial mode coverage in many cases, as shown in Figures~\ref{fig:9comp} and \ref{fig:16comp}. \\~
\textbf{Impact of Generator Learning Rate. }The above results show
that PNM can accurately cover multiple modes, however, not all
modes are \emph{fully }covered; some amount of mode collapse still
takes place. As also pointed out by, e.g.,~\cite{ACB17_WGAN},
the best response of $G$ against $\mA C$ is a single point with the
highest `realness', and therefore the WGAN they introduced uses
fewer iterations for $G$ than for $C$. Inspired by this, we investigate
if we can reduce the mode collapse by reducing the learning rate of
$G$ (to $10^{-5}$). The results in \fig~\ref{fig:slowG} clearly
show that more area of the modes are covered confirming this hypothesis.
However, it also makes clear that by doing so, we are now generating
some data outside of the true data. We point out that also with mode
collapse, the PNM mechanism theoretically could still converge, by
adding in more and more delta peaks covering parts of the modes. 
In fact, this process in work is already illustrated in the plots in \fig~\ref{fig:9comp}:
each plot in the left column contains at least one true data (black) mode which is
covered by multiple fake data (green) modes.
However, if these fake data modes would be true delta peaks, 
the number of mixture components required clearly would be
huge, making this infeasible in practice. As such, this 
leads to parameter tweaking; investigation of better solutions
is deferred to future work.

%% file: figs/fig_unifake_convergence.tex
{
\newlength{\mywOneB}
\setlength{\mywOneB}{0.7\columnwidth}
\begin{figure}[tb]
\centering
\scalebox{0.75}{
\begin{tabular}{cc}
    \hspace{-8mm}
\subfloat{\includegraphics[width = 0.98\mywOneB]{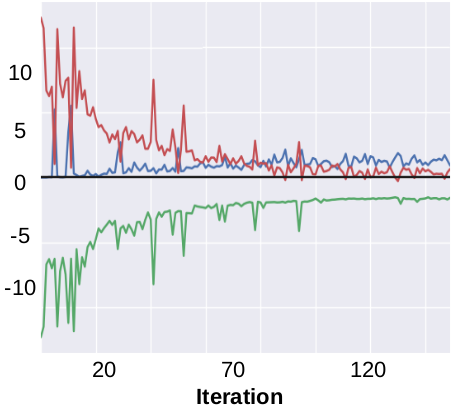}} &
\subfloat{\includegraphics[width = \mywOneB]{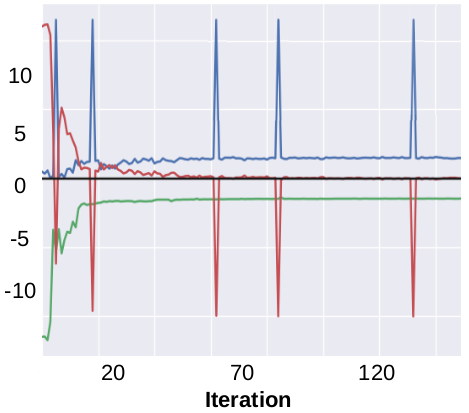}} 
\end{tabular}
}
\caption{Convergence without (left) and with (right) adding uniform fake data.
    Shown is payoff as a function of the number of iterations of PNM: generator (blue), classifier (green), tests (red).
    The tests that do not generate positive payoff (red line $< 0$) are not 
    added to the mixture.
}
\label{fig:unifakeconv}
\end{figure}
}

%% file: figs/fig2_main_comparison.tex
%
%
%

\newlength{\mywTwo}
\setlength{\mywTwo}{0.6\columnwidth}

\begin{figure*}[htbp]
\centering
\scalebox{0.85}{
\begin{tabular}{ccc}
\subfloat{\includegraphics[width = \mywTwo]{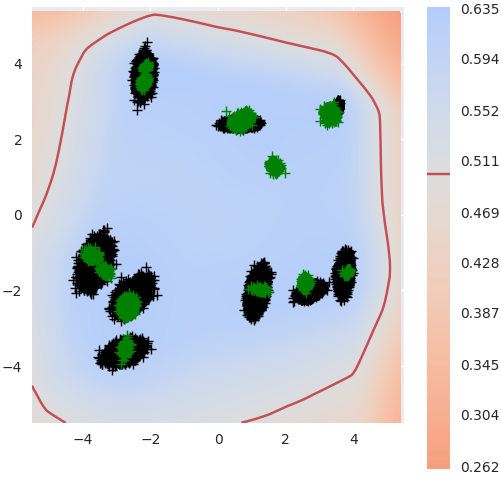}} &
\subfloat{\includegraphics[width = 1.05\mywTwo]{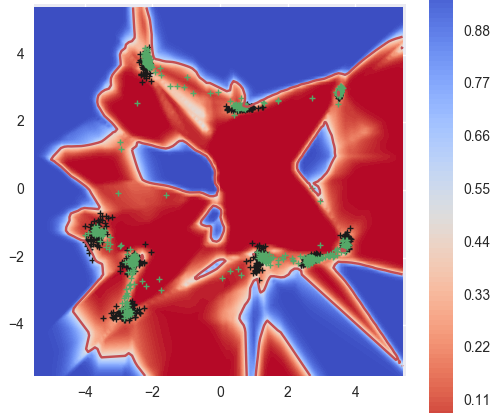}} &
\subfloat{\includegraphics[width = 1.05\mywTwo]{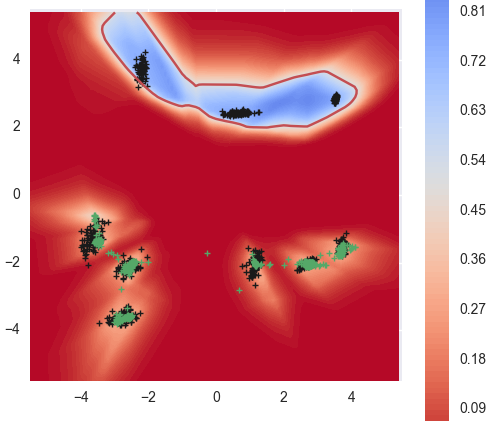}} \\

\subfloat{\includegraphics[width = \mywTwo]{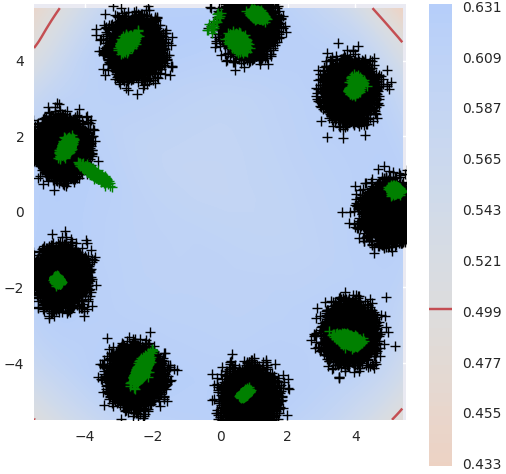}} &
\subfloat{\includegraphics[width = 1.05\mywTwo]{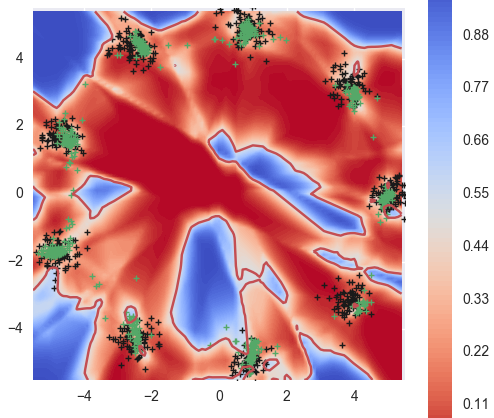}} &
\subfloat{\includegraphics[width = 1.05\mywTwo]{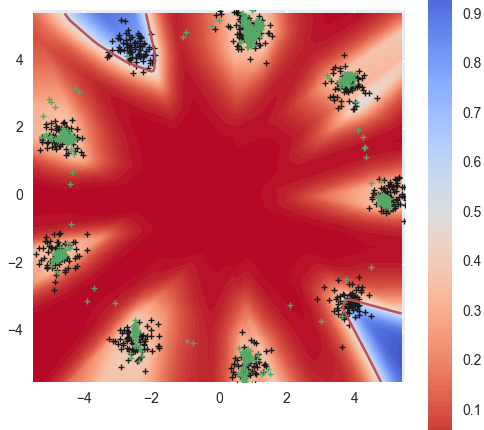}} \\

\subfloat{\includegraphics[width = \mywTwo]{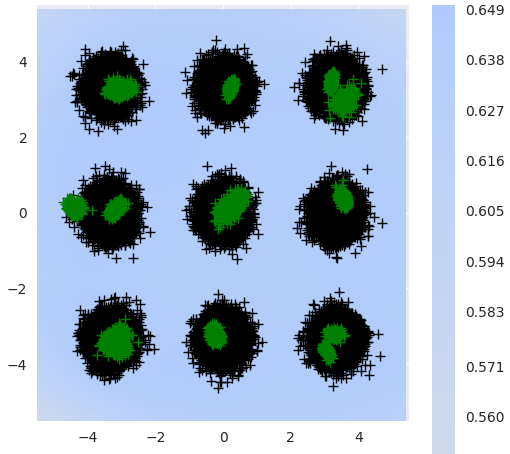}} &
\subfloat{\includegraphics[width = 1.05\mywTwo]{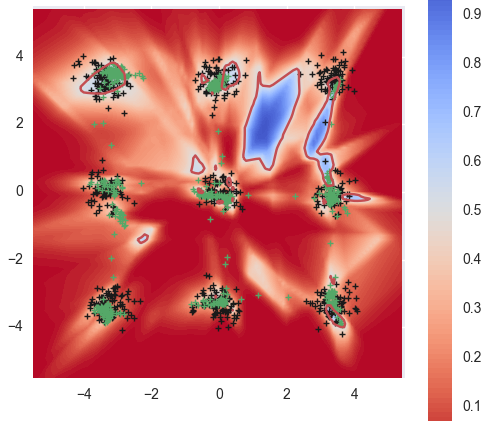}} &
\subfloat{\includegraphics[width = 1.05\mywTwo]{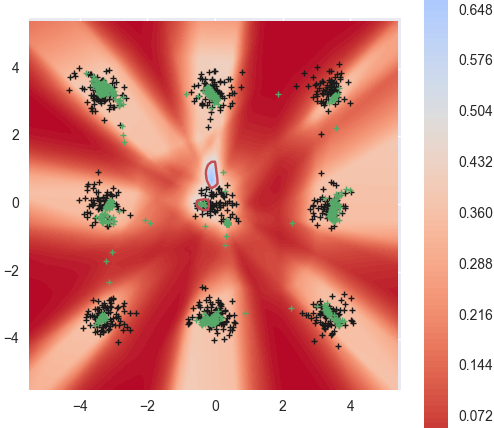}} \\

\vspace{-5mm}
\end{tabular}
}
\caption{Results on mixtures with 9 components. Left: PNM, Center: GAN with BN, Right: GAN without BN.
    True data is shown in black, while fake data is green.
    \vspace{-5mm}
}
\label{fig:9comp}
\end{figure*}

\begin{figure*}[htbp]
\centering
\scalebox{0.85}{
\begin{tabular}{ccc}
\subfloat{\includegraphics[width = \mywTwo]{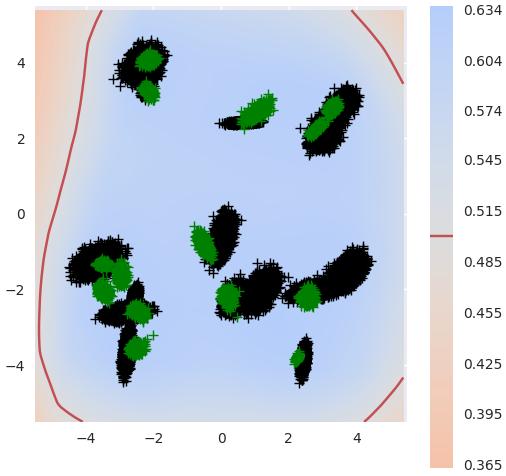}} &
\subfloat{\includegraphics[width = 1.05\mywTwo]{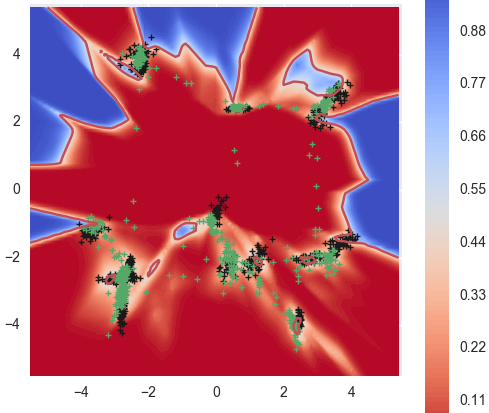}} &
\subfloat{\includegraphics[width = 1.05\mywTwo]{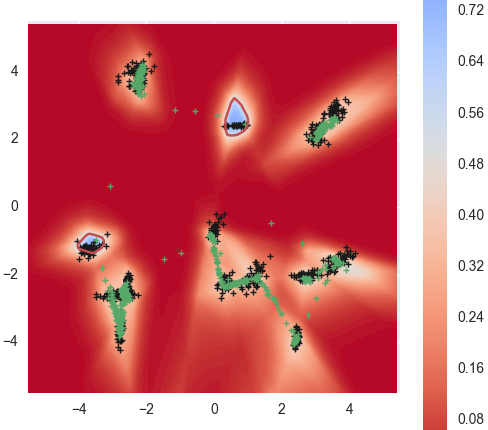}} \\

\subfloat{\includegraphics[width = \mywTwo]{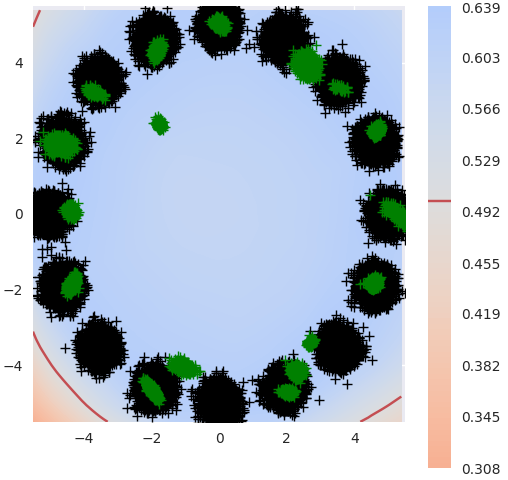}} &
\subfloat{\includegraphics[width = 1.05\mywTwo]{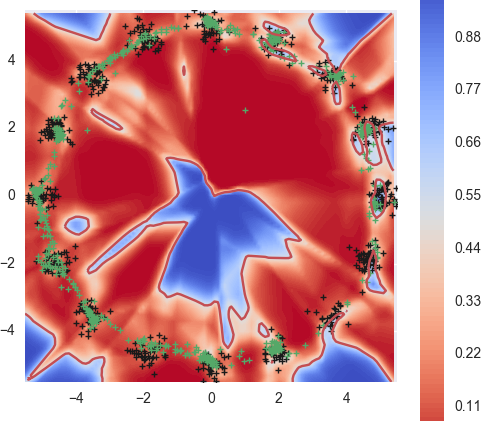}} &
\subfloat{\includegraphics[width = 1.05\mywTwo]{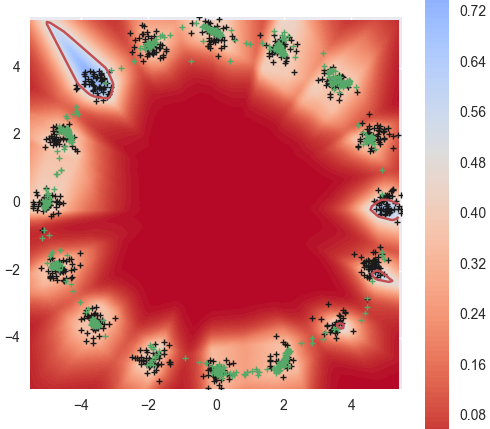}} \\

\subfloat{\includegraphics[width = \mywTwo]{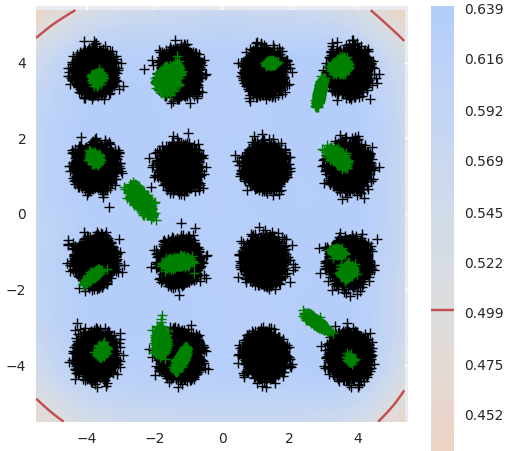}} &
\subfloat{\includegraphics[width = 1.05\mywTwo]{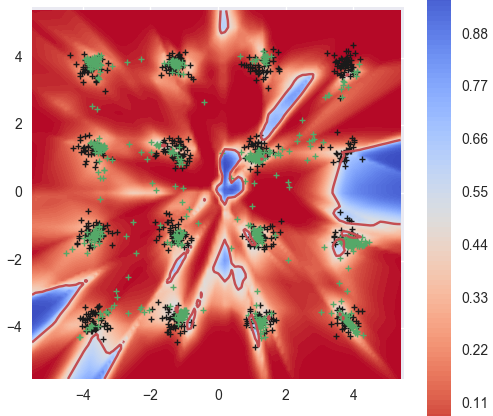}} &
\subfloat{\includegraphics[width = 1.05\mywTwo]{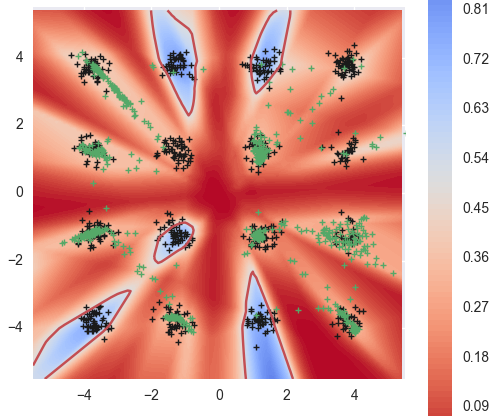}} \\

\vspace{-5mm}
\end{tabular}
}
\caption{Results on mixtures with 16 components. Left: PNM, Center: GAN with BN, Right: GAN without BN.    
    True data is shown in black, while fake data is green.
    \vspace{-5mm}
}
\label{fig:16comp}
\end{figure*}

%% file: figs/fig3_main_comparison.tex
%
%
%

\newlength{\mywThree}
\setlength{\mywThree}{0.6\columnwidth}

\begin{figure*}[htbp]
\centering
\scalebox{0.9}{
\begin{tabular}{ccc}
\subfloat{\includegraphics[width = \mywThree]{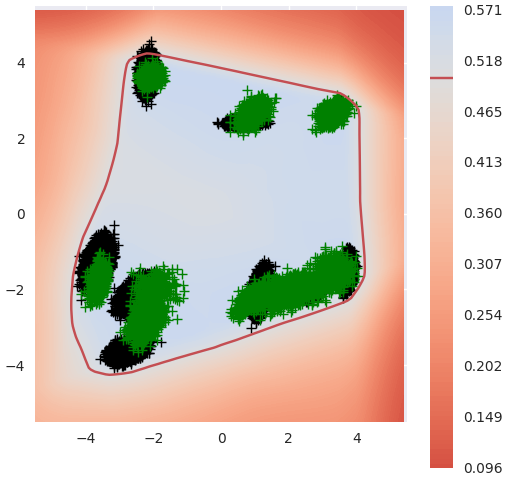}} &
\subfloat{\includegraphics[width = \mywThree]{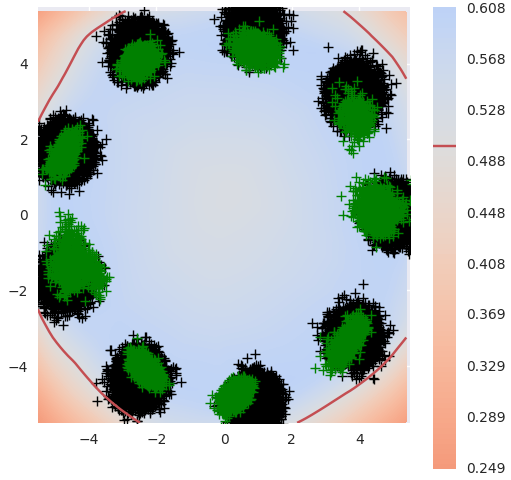}} &
\subfloat{\includegraphics[width = \mywThree]{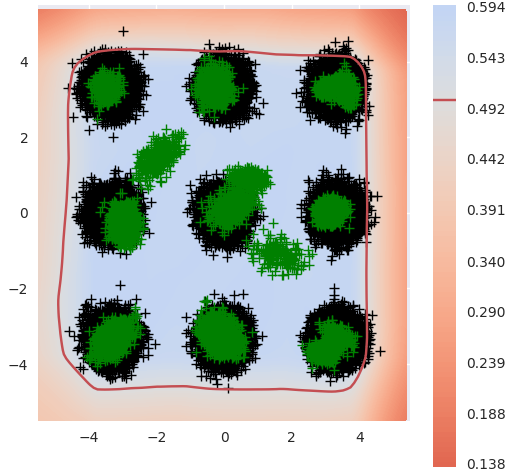}}
\end{tabular}
}
\caption{Results for PNM with a learning rate of $10^{-5}$ for the generator.}
\label{fig:slowG}
\end{figure*}

%% file: GANGs_arXiv--related-work.tex
\section{Related Work}
\label{sec:relatedWork}

\textbf{Progress in zero-sum games.} 
\citep{BosanskyKLP14} devise a double-oracle algorithm for computing exact
equilibria in extensive-form games with imperfect information.
Their algorithm uses~\emph{best response oracles}; PNM does so too, though in 
this paper using resource-bounded rather than exact best responses.
Inspired by GANs, \citep{HazanSZ17} deal with general zero-sum settings with
non-convex loss functions. They introduce a weakening of local equilibria known
as \emph{smoothed local equilibria} and provide algorithms with guarantees on
the smoothed local regret. In contrast, we work with a generalization of local
equilibrium (RB-NE) that allows for stronger notions of equilibrium, not only
weaker ones, depending on the power of one's RBBR functions.
For the more restricted class of convex-concave zero-sum games, it was recently
shown that Optimistic Mirror Descent (a variant of gradient descent)
and its generalization Optimistic Follow-the-Regularized-Leader achieve faster
convergence rates than gradient descent~\cite{RakhlinS13a,RakhlinS13b}.
These algorithms have been explored in the context of GANs by~\cite{DaskISZ17}.
However, %
the convergence results do not apply as GANs are not convex-concave.

\textbf{GANs.} 
The literature on GANs has been growing at an incredible rate, and due to space
constraints, we cannot give a full overview of all the related works, such as
\citep{ACB17_WGAN,ArjovskyB17,Huszar15_Exotic,NowozinCT16_FGAN,DaiABHC17_ICLR,ZhaoML16_EGAN_ICLR,AroraZ17,SalimansGZCRCC16,GulrajaniAADC17,Radford15arxiv}.
Instead we refer to \cite{Unterthiner17arXiv} for a comprehensive recent
overview.
That same paper also introduces Coulomb GANs \cite{Unterthiner17arXiv}.
As we do here, the authors show convergence, but for them only under the strong
assumption that the ``generator samples can move freely'' (which is not the
case when training via gradient descent; samples can only move small
steps). 
Moreover, their approach essentially performs non-parametric density
estimation, which is based on the (Euclidean) distance between data points,
which we believe undermines one of the most attractive features of GANs
(not needing a distance metric). Furthermore, it is widely known that using 
metrics in high-dimensional spaces is problematic (i.e., the ``curse
of dimensionality'', see, e.g., \cite[Sect. 2.5]{friedman2001elements}).
Perhaps not surprisingly \citeauthor{Unterthiner17arXiv} report a higher
frequency of ``generated samples that are non-sensical interpolations of
existing data modes''.

\textbf{Explicit representations of mixtures of strategies.}
Recently, more researchers have investigated the idea of (more or less)
explicitly representing a set or mixture of strategies for the players.
For instance, \cite{ImMKT16} retains sets of networks that
are trained by randomly pairing up with a network for the other player thus
forming a GAN. This, like PNM, can be interpreted as a coevolutionary
approach, but unlike PNM, it does not have any convergence guarantees.

Generally, explicit mixtures can bring advantages in two ways:
\emph{(1) Representation}: intuitively, a mixture of $k$ neural networks could better
represent a complex distribution than a single NN of the same size, and
would be roughly on par with a single network that is $k$ times as big. Arora et
al.~\cite{Arora17ICML} show how to create such a bigger network using a `multi-way
selector'. In preliminary experiments [not reported] we observed mixtures of
simpler networks leading to better performance than a single larger network. 
\emph{(2) Training}: Arora et al.\ use an architecture that is
tailored to representing a mixture of components and train a single such
network. We, in contrast, explicitly represent the mixture; given the
observation that good solutions will take the form of a mixture, this is a
form of domain knowledge that facilitates learning and convergence
guarantees. %

The most closely related paper that we have come across is by Grnarova et
al.~\cite{Grnarova17arxiv}, which also builds upon game theoretic tools to give certain
convergence guarantees. The main differences with our paper are as follows:
1) We clarify how zero-sum games relate to original GAN formulation, 
Wasserstein GAN objectives, and Goodfellow's `heuristic trick'.
2) We provide a much more general form of convergence (to an RB-NE) that is
applicable to \emph{all} architectures, that only depends on the power to compute best
responses, and show that PNM converges in this sense. We also show that if
agents can compute an $\epsilon$-best response, then the procedure converges to an $\epsilon$-NE.
3) Grnarova et al. show that for a very specific GAN architecture their
Alg.~1 converges to an $\epsilon$-NE. This result is an instantiation of our more
general theory: they assume they can compute exact (for $G$) and
$\epsilon$-approximate (for $C$) best responses; for such powerful players our
Theorem~\ref{thm:epsNashGivenSufficientResources} provides that guarantee.
4) Their Algorithm~2 %
does not provide guarantees. 

\textbf{Bounded rationality.}
The proposed notion of RB-NE is one of bounded rationality
\cite{Simon55behavioral}.
Over the years a number of different such notions have been proposed, e.g.,
see \cite{Russell97AIJ,Zilberstein11metareasoning}. 
Some of these also target agents in games. Perhaps the most
well-known such a concept is the quantal response equilibrium
\cite{McKelvey95GEB}. Other concepts take into account an explicit cost of
computation \cite{Rubinstein86JEC,Halpern14TopiCS}, or explicity limit the
allowed strategy, for instance by limiting the size of finite-state
machines that might be employed \cite{Halpern14TopiCS}.
However, these notions are motivated to explain \emph{why} people might show
certain behaviors or \emph{how} a decision maker should use its
limited resources. We on the other hand, take the why and how of bounded
rationality as a given, and merely model the outcome of a resource-bounded
computation (as the computable set $\aAS i^{RB}\subseteq\aAS i$ ). In other
words, we make a minimal assumption on the nature of the
resource-boundedness, and aim to show that even under such general assumptions we
can still reach a form of equilibrium, an RB-NE, of which the quality can
be directly linked
(via Theorem~\ref{thm:epsNashGivenSufficientResources}) 
to the computational power of the agents.

%% file: GANGs_arXiv--6-conclusions.tex
\section{Conclusions}

\label{sec:conclusions}

We introduced GANGs---Generative Adversarial Network Games---a novel
framework for %
representing adversarial generative models
by formulating them as finite zero-sum games. 
The framework provides strong links to the rich literature in game theory, and
makes available the rich arsenal of game theory solution techniques. 
It also clarifies the solutions
that we try to find as saddle points in \emph{mixed strategies} 
resolving the possibility of getting stuck in a local NE.
As finite GANGs have extremely large action spaces we cannot expect to
exactly (or $\epsilon$-approximately) solve them. Therefore,
we introduced a concept of bounded rationality,
Resource-Bounded Nash Equilibrium (RB-NE). This notion is richer than the 
`local Nash Equilibria' in that it captures not only failures of escaping
local optima of gradient descent, but applies to any approximate best
response computations, including methods with random restarts. 

While in GANGs with mixed strategies gradient descent has no inherent problems (any
local NE is a global NE), there are no known deep learning (i.e.,
gradient-based) methods to optimize mixed strategies with \emph{very} 
large support sizes. However, in our RB-NE formulation,
we can draw on a richer set of methods for solving zero-sum
games~\citep{Fudenberg98book,RakhlinS13b,Foster16NIPS,Oliehoek06GECCO}.
In this paper, we focus on PNM, which
monotonically converges to an RB-NE, and we empirically investigated
its effectiveness in search for good strategies in a setting where
the real data generating process is a mixture of Gaussians. Our proof-of-concept
results demonstrate this method indeed can deal well with typical
GAN problems such as mode collapse, partial mode coverage and forgetting.

\textbf{Future work.} We presented a framework that can have many
instantiations and modifications. 
For example, one direction is to employ different learning algorithms.
Another direction could focus on modifications of PNM, such as to allow
discarding ``stale'' pure strategies, which
would allow the process to run for longer without being inhibited by the size
of the resulting zero-sum ``subgame'' that must be maintained and repeatedly
solved.
The addition of fake uniform data as a guiding component suggests
that there might be benefit of considering ``deep-interactive learning''
where there is deepness in the number of players that interact in
order to give each other guidance in adversarial training. This could
potentially be modelled by zero-sum polymatrix games~\citep{CaiCDP16}.

%% file: ms.bbl

\providecommand{\noopsort}[1]{}